\documentclass{article}
\usepackage[utf8]{inputenc}
\pdfoutput=1
\usepackage[sc,osf]{mathpazo}
\usepackage[height=8.5in,width=6in,letterpaper]{geometry}
\usepackage[sort&compress,numbers]{natbib}
\usepackage[colorlinks=true,citecolor=blue,breaklinks]{hyperref}
\usepackage[hyphenbreaks]{breakurl} 
\usepackage[tracking]{microtype}

\makeatletter
\newlength\aftertitskip     \newlength\beforetitskip
\newlength\interauthorskip  \newlength\aftermaketitskip

\def\maketitle{\par
 \begingroup
   \def\thefootnote{\fnsymbol{footnote}}
   \def\@makefnmark{\hbox to 4pt{$^{\@thefnmark}$\hss}}
   \@maketitle \@thanks
 \endgroup
\setcounter{footnote}{0}
 \let\maketitle\relax \let\@maketitle\relax
 \gdef\@thanks{}\gdef\@author{}\gdef\@title{}\let\thanks\relax}

\def\@startauthor{\noindent \normalsize\bf}
\def\@endauthor{}
\def\@starteditor{\noindent \small {\bf Editor:~}}
\def\@endeditor{\normalsize}
\def\@maketitle{\vbox{\hsize\textwidth
 \linewidth\hsize \vskip \beforetitskip
 {\begin{center} \LARGE\@title \par \end{center}} \vskip \aftertitskip
 {\def\and{\unskip\enspace{\rm and}\enspace}%
  \def\addr{\small\it}%
  \def\email{\hfill\small\tt}%
  \def\name{\normalsize\bf}%
  \def\AND{\@endauthor\rm\hss \vskip \interauthorskip \@startauthor}
  \@startauthor \@author \@endauthor}
}}

\makeatother

\pdfoutput=1                    

\usepackage{floatrow} 
\newfloatcommand{capbtabbox}{table}[][\FBwidth]

\usepackage{graphicx}
\usepackage{amsmath,amsthm,amssymb,bm} 
\usepackage{amsfonts}
\usepackage{mathrsfs}
\usepackage{subfigure}
\usepackage{xspace}
\usepackage{array}
\usepackage{enumerate}
\usepackage{algorithm}
\usepackage{algorithmic}
\usepackage{stmaryrd}
\usepackage{appendix}
\usepackage{wrapfig}
\numberwithin{equation}{section}

  \newcommand{\cc}{\mathcal{C}}                     \newcommand{\cy}{\mathcal{Y}} 

\newcommand{\pr}{\mathbb{P}} 

\newcommand{\argmax}{\textrm{argmax}}
\newcommand{\argmin}{\textrm{argmin}}

\theoremstyle{plain}
\newtheorem{theorem}{Theorem}

\newtheorem{lemma}[theorem]{Lemma}

\theoremstyle{definition}

\theoremstyle{remark}

\newcommand{\refalgo}[1]{Algo.~\ref{#1}}
\newcommand{\refsec}[1]{Sec.~\ref{#1}}
\newcommand{\reffig}[1]{Fig.~\ref{#1}}

\newcommand{\refthm}[1]{Thm.~\ref{#1}}

\newcommand{\Ph}{\widehat{P}_k}
\newcommand{\PC}{P_{\cc,k}}
\newcommand{\Lc}{\widetilde{L}}
\newcommand{\Yt}{\widetilde{Y}}

\newcommand{\Pp}{P'}

\newcommand{\pns}{\delta_{\Pi}} 
\newcommand{\epsm}{\varepsilon_{\Pi}} 

\DeclareMathOperator{\sspan}{span}
\newcommand{\dpp}{\textsc{Dpp}\xspace}
\newcommand{\kdpp}[1]{\textsc{Dpp}_k(#1)}
\newcommand{\tvnorm}[1]{\|#1\|_{\text{\rm tv}}}
\newcommand{\algo}{\textsc{CoreDpp}\xspace}
\newcommand{\algoz}{\textsc{CoreDpp-z}\xspace}
\newcommand{\algor}{\textsc{CoreDpp-r}\xspace}

\title{Efficient Sampling for $k$-Determinantal\\ Point Processes}
\author{\name Chengtao Li \email{ctli@mit.edu}\\
  \addr{Massachusetts Institute of Technology}
  \AND 
  \name Stefanie Jegelka \email{stefje@csail.mit.edu}\\
  \addr{Massachusetts Institute of Technology} 
  \AND
  Suvrit Sra \email{suvrit@mit.edu}\\
  \addr{Massachusetts Institute of Technology}
}

\begin{document}
\maketitle

\begin{abstract}
  Determinantal Point Processes (\dpp{}s) are elegant probabilistic models of repulsion and diversity over discrete sets of items. But their applicability to large sets is hindered by expensive cubic-complexity matrix operations for basic tasks such as sampling. In light of this, we propose a new method for approximate sampling from discrete $k$-\dpp{}s. Our method takes advantage of the diversity property of subsets sampled from a \dpp, and proceeds in two stages: first it constructs coresets for the ground set of items; thereafter, it efficiently samples subsets based on the constructed coresets. As opposed to previous approaches, our algorithm aims to minimize the total variation distance to the original distribution. Experiments on both synthetic and real datasets indicate that our sampling algorithm works efficiently on large data sets, and yields more accurate samples than previous approaches.
\end{abstract}

\section{Introduction}
\label{sec:introduction}
Subset selection problems lie at the heart of many applications where a small subset of items must be selected to represent a larger population. Typically, the selected subsets are expected to fulfill various criteria such as sparsity, grouping, or diversity. Our focus is on \emph{diversity}, a criterion that plays a key role in a variety of applications, such as gene network subsampling~\cite{Batmanghelich2014}, document summarization~\cite{lin2011}, video summarization~\cite{Gong2014}, content driven search~\citep{Fox2014}, recommender systems~\citep{zhou2010solving}, sensor placement~\citep{krause2008}, among many others~\cite{Agarwal2009,Kulesza2010,Gillenwater2012a,Kulesza2012,Affandi2012,Snoek2013,Shah2013}.

Diverse subset selection amounts to sampling from the set of all subsets of a ground set according to a measure that places more mass on subsets with qualitatively different items. An elegant realization of this idea is given by Determinantal Point Processes (\dpp{}s), which are probabilistic models that capture diversity by assigning subset probabilities proportional to (sub)determinants of a kernel matrix. 

\dpp{}s enjoy rising interest in machine learning~\cite{Kulesza2010,Kulesza2011a,Zou2012,Kang2013,Fox2014,Gong2014,Mariet2015}; a part of their appeal can be attributed to computational tractability of basic tasks such as  computing partition functions, sampling, and extracting marginals~\cite{Hough2005,Kulesza2012}. But despite being polynomial-time, these tasks remain infeasible for large data sets. \dpp sampling, for example, relies on an eigendecomposition of the \dpp kernel, whose cubic complexity is a huge impediment. Cubic preprocessing costs also impede wider use of the cardinality constrained variant $k$-\dpp{}~\citep{Kulesza2011a}. 

These drawbacks have triggered work on approximate sampling methods. Much work has been devoted to approximately sample from a \dpp by first approximating its kernel via algorithms such as the Nystr\"om method~\cite{Kulesza2013}, Random Kitchen Sinks~\cite{Rahimi2007,Affandi2013}, or matrix ridge approximations~\cite{Zhang2014,Wang2014}, and then sampling based on this approximation. However, these methods are somewhat inappropriate for sampling because they aim to project the \dpp kernel onto a lower dimensional space while minimizing a matrix norm, rather than minimizing an error measure sensitive to determinants. Alternative methods use a dual formulation \cite{Kulesza2010}, which however presupposes a decomposition $L=XX^\top$ of the DPP kernel, which may be unavailable and inefficient to compute in practice. Finally, MCMC \cite{Kang2013,decreusefond2013perfect,belabbas2009spectral,anari16} offers a potentially attractive avenue different from the above approaches that all rely on the same spectral technique.

We pursue a yet different approach. While being similar to matrix approximation methods in exploiting redundancy in the data, in sharp contrast to methods that minimize matrix norms, we focus on minimizing the total variation distance between the original \dpp and our approximation. As a result, our approximation models the true \dpp probability distribution more faithfully, while permitting faster sampling. We make the following key contributions:

\begin{list}{--}{\leftmargin=1.5em}
\setlength{\itemsep}{0pt}
\item An algorithm that constructs coresets for approximating a $k$-\dpp by exploiting latent structure in the data. The construction, aimed at minimizing the total variation distance, takes $O(NM^3)$ time; linear in the number $N$ of data points. The construction works as the overhead of sampling algorithm and is much faster than standard cubic-time overhead that exploits eigendecomposition of kernel matrices. We also investigate conditions under which such an approximation is good.
\item A sampling procedure that yields approximate $k$-\dpp subsets using the constructed coresets. While most other sampling methods sample diverse subsets in $O(k^2N)$ time, the sampling time for our coreset-based algorithm is $O(k^2M)$, where $M\ll N$ is a user-specified parameter \emph{independent of $N$}.
\end{list}

Our experiments indicate that our construction works well for a wide range of datasets, delivers more accurate approximations than the state-of-the-art, and is more efficient, especially when multiple samples are required.

\paragraph{Overview of our approach.}
Our sampling procedure runs in two stages. Its first stage constructs an approximate probability distribution close in total variation distance to the true $k$-\dpp. The next stage efficiently samples from this approximate distribution.

Our approximation is motivated by the diversity sampling nature of \dpp{}s: in a \dpp most of the probability mass will be assigned to diverse subsets. This leaves room for exploiting redundancy. In particular, if the data possesses latent grouping structure, certain subsets will be much more likely to be sampled than others.  For instance, if the data are tightly clustered, then any sample that draws two points from the same cluster will be very unlikely. 

The key idea is to reduce the effective size of the ground set. We do this via the idea of coresets~\cite{Har2004,Feldman2006}, small subsets of the data that capture function values of interest almost as well as the full dataset. Here, the function of interest is a $k$-\dpp distribution. Once a coreset is constructed, we can sample a subset of core points, and then, based on this subset, sample a subset of the ground set. For a coreset of size $M$, our sampling time is $O(k^2M)$, which is \emph{independent of $N$} since we are using $k$-\dpp{}s~\citep{Kulesza2011a}. 

\paragraph{Related work.} \dpp{}s have been studied in statistical physics and probability~\cite{Hough2005,Borodin2005,Borodin2009}; they have witnessed rising  interest in machine learning  \cite{Kulesza2012,Kulesza2011a,Kulesza2010,Gillenwater2012a,Snoek2013,Gong2014,Zou2012,Mariet2015}. Cardinality-conditioned \dpp sampling is also referred to as ``volume sampling'', which has been used for matrix approximations \cite{Deshpande2006,Deshpande2010}. 
Several works address faster \dpp sampling via matrix approximations~\cite{Kulesza2010,Magen2008,Deshpande2010,Kulesza2013} or MCMC~\cite{belabbas2009spectral,Kang2013}. 
Except for MCMC, even if we exclude preprocessing, known sampling methods still require $O(k^2N)$ time for a single sample; we reduce this to $O(k^2M)$. Finally, different lines of work address learning DPPs~\cite{Kulesza2011,Gillenwater2014,Fox2014,Mariet2015} and MAP estimation \cite{Gillenwater2012}.

Coresets have been applied to large-scale clustering \cite{har2004coresets,har2005smaller,feldman2013turning,bachem2015coresets}, PCA and CCA~\cite{feldman2013turning,paul2015core}, and segmentation of streaming data~\cite{rosman2014coresets}.

\section{Setup and basic definitions}
A determinantal point process $\dpp(L)$ is a distribution over all subsets of a ground set $\cy$ of cardinality $N$. It is determined by a positive semidefinite kernel $L\in\mathbb{R}^{N\times N}$. Let $L_Y$ be the submatrix of $L$ consisting of the entries $L_{ij}$ with $i,j \in Y \subseteq \cy$. Then, the probability $P_L(Y)$ of observing $Y \subseteq \cy$ is proportional to $\det(L_Y)$; consequently, $P_L(Y) = \det(L_Y)/\det(L+I)$.
Conditioning on sampling sets of fixed cardinality $k$, one obtains a $k$-\dpp~\cite{Kulesza2011a}:
\begin{align*}
  P_{L,k}(Y) :&= P_L(Y \mid\; |Y|=k) \\
  &= \det(L_Y)e_k(L)^{-1}\llbracket\, |Y|=k\rrbracket,
\end{align*}
where $e_k(L)$ is the $k$-th coefficient of the characteristic polynomial $\det(\lambda I - L) = \sum_{k=0}^N(-1)^ke_k(L)\lambda^{N-k}$. 
We assume that $P_{L,k}(Y) > 0$ for all subsets $Y \subseteq \cy$ of cardinality $k$. To simplify notation, we also write $P_k \triangleq P_{L,k}$.

Our goal is to construct an approximation $\Ph$ to $P_k$ that is close in \emph{total variation distance}
\begin{equation}
  \label{eq:tv}
  \tvnorm{\Ph-P_k} := \tfrac{1}{2}\sum_{Y \subseteq \cy, |Y|=k}|\Ph(Y) - P_k(Y)|,
\end{equation}
and permits faster sampling than $P_k$. Broadly, we proceed as follows.
First, we define a partition $\Pi = \{\cy_1, \ldots, \cy_M\}$ of $\cy$ and 
extract a subset $\cc \subset \cy$ of $M$ core points, containing one point from each part. Then, for the set $\cc$ we construct a special kernel $\Lc$ (as described in Section~\ref{sec:sampling}).  
When sampling, we first sample a set $\Yt \sim \kdpp{\Lc}$ and then, for each $c \in \Yt$ we uniformly sample one of its assigned points $y \in \cy_c$. These second-stage points $y$ form our final sample. We denote the resulting distribution by $\Ph = \PC$. Algorithm~\ref{smplalgo} formalizes the sampling procedure, which, after one eigendecomposition of the small matrix $\Lc$.  

We begin by analyzing the effect of the partition on the approximation error, and then devise an algorithm to approximately minimize the error. We empirically evaluate our approach in Section~\ref{sec:exp}.

\section{Coreset sampling}

Let $\Pi = \{\cy_1,\ldots, \cy_M\}$ be a \textit{partition} of $\cy$, i.e., $\cup_{i=1}^M \cy_i = \cy$ and $\cy_i\cap \cy_j = \emptyset$ for $i\neq j$. We call $\cc\subseteq \cy$ a \textit{coreset} with respect to a partition $\Pi$ if $|\cc \cap \cy_i| = 1$ for $i\in[M]$. With a slight abuse of notation, we index each part $\cy_c \in \Pi$ by its core $c \in \cc \cap \cy_c$. Based on the partition $\Pi$, we call a set $Y \subseteq \cy$ \emph{singular}\footnote{In combinatorial language, $Y$ is an independent set in the partition matroid defined by $\Pi$.} with respect to $\Pi'\subseteq\Pi$,  if for $\cy_i\in \Pi'$ we have $|Y \cap \cy_i| \le 1$ and for $\cy_j\in\Pi\backslash \Pi'$ we have $|Y \cap \cy_j| = 0$. We say $Y$ is $k$-singular if $Y$ is singular and $|Y|=k$. 

\label{sec:sampling}
Given a partition $\Pi$ and core $\cc$, we construct a rescaled core kernel $\Lc \in \mathbb{R}^{M \times M}$ with entries $\Lc_{c,c'} = \sqrt{|\cy_c||\cy_{c'}|}L_{c,c'}$. We then use this smaller matrix $\Lc$ and its eigendecomposition as an input to our two-stage sampling procedure in Algorithm~\ref{smplalgo}, which we refer to as \algo. The two stages are: (i) sample a $k$-subset from $\cc$ according to $\kdpp{\Lc}$; and (ii) for each $c$, pick an element $y \in \cy_c$ uniformly at random. 
This algorithm uses only the much smaller matrix $\Lc$ and samples a subset from $\cy$ in $O(k^2M)$ time. When $M\ll N$ and we want many samples, it translates into a notable improvement over the $O(k^2N)$ time of sampling directly from $\kdpp{L}$. 

The following lemma shows that \algo is equivalent to sampling from a $k$-\dpp where we replace each point in $\cy$ by its corresponding core point, and sample with the resulting induced kernel $L_{\cc(\cy)}$.
\begin{lemma}
  \label{lem:lemma_smpl.main}
  \algo is equivalent to sampling from $\kdpp{L_{\cc(\cy)}}$, where in $L_{\cc(\cy)}$ we replace each element in $\cy_c$ by $c$, for all $c \in \cc$.
\end{lemma}
\begin{proof}
  We denote the distribution induced by~\refalgo{smplalgo} by $\PC$ and that induced by $\kdpp{L_{\cc(\cy)}}$ by $\Pp_k$.
  
  First we claim that both sampling algorithms can only sample $k$-singular subsets. By construction, $\PC$ picks one or zero elements from each $\cy_{c}$. For $\Pp_k$, if $Y$ is $k$-nonsingular, then there would be identical rows in $(L_{\cc(\cy)})_Y = L_{\cc(Y)}$, resulting in $\det(L_{\cc(Y)}) = 0$. Hence both $\PC$ and $\Pp_k$ only assign nonzero probability to $k$-singular sets $Y$. As a result, we have
  \begin{align*}
  \label{eq:normalization}
  e_k&(L_{\cc(\cy)}) = \sum_{C\textit{ is $k$-singular}} (\prod_{c\in C}|\cy_c|)\det(L_C)\nonumber\\
  =&\sum_{C\subseteq\cc,|C|=k}(\prod_{c\in C}|\cy_c|\det(L_C))
  = \sum_{|C|=k} \det(\Lc_C) = e_k(\Lc).
  \end{align*}
  For any $Y=\{y_1,\ldots, y_k\}\subseteq \cy$ that is $k$-singular, we have
  \begin{align*}
  \PC(Y) &= {\det(\Lc_{\cc(Y)}) \over e_k(\Lc)\prod_{i=1}^k |\cy_{\cc(y_i)}|} = {(\prod_{i=1}^k |\cy_{\cc(y_i)}|)\det(L_{\cc(Y)})\over e_k(L_{\cc(\cy)})\prod_{i=1}^k |\cy_{\cc(y_i)}|} \\
  &= {\det(L_{\cc(Y)})\over e_k(L_{\cc(\cy)})} = \Pp_k(Y),
  \end{align*} 
  which shows that these two distributions are identical, i.e., sampling from $\kdpp{\Lc}$ followed by uniform sampling is equivalent to directly sampling from $\kdpp{L_{\cc(\cy)}}$.
  \end{proof}

\begin{algorithm}[t]
  \caption{\algo Sampling}\label{smplalgo}
  \begin{algorithmic} 
    \STATE \textbf{Input:} core kernel $\Lc \in \mathbb{R}^{M \times M}$ and its eigendecomposition; partition $\Pi$; size $k$
    \STATE sample $C\sim \kdpp{\Lc}$
    \STATE sample $y_i\sim \mathrm{Uniform}(\cy_{c})$ for $c\in C$
    \RETURN $Y = \{y_1, \ldots, y_k\}$
  \end{algorithmic}
\end{algorithm}

\section{Partition, distortion and approximation error}
Let us provide some insight on quantities that affect the distance $\tvnorm{\PC - P_k}$ when sampling with~\refalgo{smplalgo}. In a nutshell, this distance depends on three key quantities (defined below): the probability of nonsingularity $\pns$, the distortion factor $1+\epsm$, and the normalization factor. 

For a partition $\Pi$ we define the \emph{nonsingularity probability} $\pns$ as the probability that a draw $Y \sim \kdpp{L}$ is not singular 
with respect to~any $\Pi'\subseteq\Pi$. 

Given a coreset $\cc$, we define the \emph{distortion factor} $1+\epsm$ (for $\epsm \ge 0$) as a partition-dependent quantity, so that for any $c \in \cc$, for all $u,v \in \cy_c$, and for any $(k-1)$-singular set $S$ with respect to $\Pi \setminus \cy_c$ the following bound holds:
\begin{equation}\label{eq:defeps}
  \frac{\det(L_{S \cup \{u\}})}{\det(L_{S \cup \{v\}})} = \frac{L_{u,u} - L_{u,S}L^{-1}_SL_{S,u}}{L_{v,v} - L_{v,S}L^{-1}_SL_{S,v}} \leq 
  1 + \epsm.
\end{equation}
If $\phi$ is the feature map corresponding to the kernel $L$, then geometrically, the numerator of~\eqref{eq:defeps} is the length of the projection of $\phi(u)$ onto the orthogonal complement of $\sspan\{\phi(s) \mid s \in S\}$.

The \emph{normalization factor} for a $k$-\dpp($L$) is simply $e_k(L)$.

Given $\Pi$, $\cc$ and the corresponding nonsingularity probability and distortion factors, we have the following bound:
\begin{lemma}
  \label{lem:epsk.main}
  Let $Y\sim\kdpp{L}$ and $\cc(Y)$ be the set where we replace each $y \in Y$ by its core $c\in\cc$, i.e., $y\in\cy_c$. 
  With probability $1-\pns$, it holds that
  \begin{align}
    (1+\epsm)^{-k}\le \frac{ \det(L_{\cc(Y)})}{\det(L_{Y})} \le (1 + \epsm)^{k}.
  \end{align}
\end{lemma}
  \begin{proof}
  Let $c \in \cc$ and consider any $(k-1)$-singular set $S$ with respect to $\Pi \setminus \cy_c$. Then, for any $v \in \cy_c$, using Schur complements and by the definition of $\epsm$ we see that 
  \begin{align*}
    (1+\epsm)^{-1}&\le \frac{\det(L_{S \cup \{c\}})}{\det(L_{S \cup \{v\}})} = \frac{L_{c,c} - L_{c,S}L^{-1}_SL_{S,c}}{L_{v,v} - L_{v,S}L^{-1}_SL_{S,v}} \\
    &= \frac{||Q_{S^\perp}\phi(c)||^2}{||Q_{S^\perp}\phi(v)||^2} \leq (1+\epsm).  
  \end{align*}
  Here, $Q_{S^\perp}$ is the projection onto the orthogonal complement of $\sspan\{\phi(s) \mid s \in S\}$, and $\phi$ the feature map corresponding to the kernel $L$.

  With a minor abuse of notation, we denote by $\cc(y) = c$ the core point corresponding to $y$, i.e., $y\in\cy_c$. For any $Y = \{y_1,\ldots, y_k\}$, we then define the sets $Y_i = \{\cc(y_1),\ldots, \cc(y_i), y_{i+1},\ldots, y_k\}$, where we gradually replace each point by its core point, with $Y_0 = Y$. If $Y$ is $k$-singular, then $\cc(y_i)\ne \cc(y_j)$ whenever $i\ne j$, and, for any $0 \leq i \leq k-1$, it holds that
  \begin{align*}
  (1+\epsm)^{-1} \le {\det(L_{Y_{i+1}})\over \det(L_{Y_i})} \le 1+\epsm.
  \end{align*}
  Hence we have
  \begin{align*}
    (1+\epsm)^{-k} \le {\det(L_{\cc(Y)})\over \det(L_{Y})} 
  = \prod_{i = 0}^{k-1} {\det(L_{Y_{i+1}})\over \det(L_{Y_i})} \le (1+\epsm)^{k}.
  \end{align*}
  This bound holds when $Y$ is $k$-singular, and, by definition of $\pns$, this happens with probability $1-\pns$.
 \end{proof}

Assuming $\epsm$ is small, Lemma~\ref{lem:epsk.main} states that if replacing a single element in a given subset with another one in the same part does not cause much distortion, then replacing all elements in the subset with their corresponding cores will cause little distortion. This observation is key to our approximation: if we can construct such a partition and coreset, we can safely replace all elements with core points and then approximately sample with little distortion. More precisely, we then obtain the following result that bounds the variational error. Our subsequent construction aims to minimize this bound.

\begin{theorem}\label{thm:varbound.main}
  Let $P_k=\kdpp{L}$ and let $\PC$ be the distribution induced by~\refalgo{smplalgo}. With the normalization factors $Z = e_k(L)$ and $Z_\cc = e_k(\Lc)$, the total variation distance between  $P_k$ and $\PC$ is bounded by
  \begin{align*}
    \tvnorm{P_k - \PC} &\le |1-\tfrac{Z_\cc}{Z}| + k\epsm + (1-k\epsm)\pns.
  \end{align*}
\end{theorem}
  \begin{proof}
  From the definition of $Z$ and $Z_\cc$ we know that $Z = \sum_{|Y|=k}\det(L_Y)$ and
  \begin{align*}
  Z_\cc &= \sum_{|Y|=k}\det((L_{\cc(\cy)})_Y)
  = \sum_{|Y|=k}\det(L_{\cc(Y)})\\
  &= \sum_{Y \text{ $k$-singular}}\det(L_{\cc(Y)}).
  \end{align*}
  The last equality follows since, as argued above, $\det(L_{\cc(Y)}) = 0$ for nonsingular $Y$. 
  It follows that
  \begin{align}
    \nonumber
  &\tvnorm{P_k - \PC}
  = \sum_{|Y|=k} |P_k(Y) - \PC(Y)|\\
  \label{eq:twoterms}
  &= \sum_{Y \text{ $k$-singular}} |P_k(Y) - \PC(Y)| + \sum_{Y\text{ $k$-nonsingular}}P_k(Y).
  \end{align}
  For the first term, we have
  \begin{align*}
  &\sum_{Y \text{ $k$-singular}} |P_k(Y) - \PC(Y)| = \sum_{Y\text{ $k$-singular}} \Big|{\det(L_Y)\over Z} - {\det(L_{\cc(Y)})\over Z_\cc}\Big|\\
  &\le \sum_{Y \text{ $k$-singular}}\Big|{1\over Z}(\det(L_Y) - \det(L_{\cc(Y)}))\Big| + \sum_{Y \text{ $k$-singular}}\Big|\det(L_{\cc(Y)})\Big({1\over Z} - {1\over Z_\cc}\Big)\Big|\\
  &= {1\over Z}\sum_{Y \text{ $k$-singular}} \det(L_Y)\Big|1 - {\det(L_{\cc(Y)})\over \det(L_Y)}\Big| + Z_\cc \Big|{1\over Z}-{1\over Z_\cc}\Big|\\
  &\le k\epsm (1-\pns)  + \Big|1-{Z_\cc\over Z}\Big|,
  \end{align*}
  where the first inequality uses the triangle inequality and the second inequality relies on Lemma~\ref{lem:epsk.main}. For the second term in \eqref{eq:twoterms}, we use that, by definition of $\pns$,
  \begin{align*}
  \sum_{Y\text{ $k$-nonsingular}}P_k(Y) &= \pns.
  \end{align*}
  Thus the total variation difference is bounded as
  \begin{align*}
  \tvnorm{P_k - \PC} &\le \Big|1-{Z_\cc\over Z}\Big| + k\epsm(1-\pns)  + \pns\\
  &= \Big|1-{Z_\cc\over Z}\Big| + k\epsm  + (1-k\epsm)\pns.\qedhere
  \end{align*}
  \end{proof}

In essence, if the probability of nonsingularity and the distortion factor are low, then it is possible to obtain a good coreset approximation. This holds, for example, if the data has intrinsic (grouping) structure. In the next subsection we provide further intuition on when we can achieve low error.

\subsection{Sufficient conditions for a good bound}
Theorem~\ref{thm:varbound.main} depends on the data and the partition $\Pi$. 
Here, we aim to  obtain some further intuition on the properties of $\Pi$ that govern the bound. At the same time, these properties suggest sufficient conditions for a ``good'' coreset $\cc$.
For each $\cy_c$, we define the diameter
\begin{align}\label{eq:diam}
  \rho_c := \max_{u,v\in\cy_c} \sqrt{L_{uu} + L_{vv} - 2 L_{uv}}.
\end{align} 
Next, define the minimum distance of any point $u \in \cy_c$ to the subspace spanned by the feature vectors of points in a ``complementary'' set $S$ that is singular with respect to\ $\Pi\setminus \cy_c$: 
\begin{align*}
  d_c := \min_{S,u} \sqrt{\tfrac{\det(L_{S\cup\{u\}})}{\det(L_S)}} 
  = \min_{S,u}\sqrt{L_{u,u} - L_{u,S}L_{S}^{-1}L_{S,u}}.
\end{align*}
Lemma~\ref{lem:boundeps.main} connects these quantities with $\epsm$; it essentially poses a separability condition on $\Pi$ (i.e., $\Pi$ needs to be ``aligned'' with the data) so that the bound on $\epsm$ holds.

\begin{lemma}\label{lem:boundeps.main}
  If $d_c > \rho_c$ for all $c \in \cc$, then
  \begin{align}
    \epsm \le \max_{c \in \cc} {(2d_c-\rho_c)\rho_c\over (d_c-\rho_c)^2}.
  \end{align}
\end{lemma}
  \begin{proof} 
  For any $c\in\cc$ and any $u,v\in\cy_c$ and $S$ $(k-1)$-singular with respect to $\Pi\backslash \cy_c$, we have
  \begin{align*}
  {\det(L_{S\cup\{u\}})\over\det(L_{S\cup\{v\}})} &= {\det(L_S)(L_{u,u} - L_{u,S}L_{S}^{-1}L_{S,u}) \over \det(L_S)(L_{v,v} - L_{v_i,S}L_{S}^{-1}L_{S,v_i})}\\
  &= {L_{u,u} - L_{u,S}L_{S}^{-1}L_{S,u} \over L_{v,v} - L_{v,S}L_{S}^{-1}L_{S,v}} = {\|Q_{S^\perp}\phi(u)\|^2\over \|Q_{S^\perp}\phi(v)\|^2}.
  \end{align*}
  Without loss of generality,
  we assume $\det(L_{S\cup\{u\}}) \ge \det(L_{S\cup\{v\}})$.
  By definition of $\rho_c$ we know that 
  \begin{align*}
  0&\le \|Q_{S^\perp}\phi(u)\| - \|Q_{S^\perp}\phi(v)\| \\
  &\le \|Q_{S^\perp}(\phi(u)-\phi(v))\| \le \|\phi(u)-\phi(v)\|\le \rho_c.
  \end{align*}
  Since $0 < \|Q_{S^\perp}\phi(v)\|\le \|Q_{S^\perp}\phi(u)\|\le \|\phi(u)\|$ by assumption, we have
  \begin{align*}
  &{\|Q_{S^\perp}\phi(u)\|^2 \over \|Q_{S^\perp}\phi(v)\|^2} \le {\|Q_{S^\perp}\phi(u)\|^2 \over (\|Q_{S^\perp}\phi(u)\| - \rho_c)^2}\\
  &\le \Big({\|\phi(u)\|\over \|\phi(u)\| - \rho_c}\Big)^2
  \le \Big({d_c\over (d_c-\rho_c)}\Big)^2.
  \end{align*}
  Then, by definition of $\epsm$, we have
  \begin{align*}
  1+\epsm\le \max_c {d_c^2\over (d_c-\rho_c)^2},
  \end{align*}
  from which it follows that
  \begin{equation*}
  \epsm \le \max_c {(2d_c-\rho_c)\rho_c\over (d_c-\rho_c)^2}.\qedhere
  \end{equation*}
  \end{proof}

\section{Efficient construction}
\label{sec:algorithms}
\refthm{thm:varbound.main} states an upper bound on the error induced by \algo and relates the total variation distance to $\Pi$ and $\cc$. Next, we explore how to efficiently construct $\Pi$ and $\cc$ that approximately minimize the upper bound.

\subsection{Constructing  $\Pi$}
Any set $Y$ sampled via \algo is, by construction, singular with respect to $\Pi$. In other words, \algo assigns zero mass to any nonsingular set. Hence, we wish to construct a partition $\Pi$ such that its nonsingular sets have low probability under $\kdpp{L}$. The optimal such partition minimizes the probability $\pns$ of nonsingularity.  A small $\pns$ value also means that the parts of $\Pi$ are dense and compact, i.e., the diameter $\rho_c$ in Equation~\eqref{eq:diam} is small.

Finding such a partition optimally is hard, so we resort to local search.
Starting with a current partition $\Pi$, we re-assign each $y$ to a part $\cy_c$ to minimize $\pns$.
If we assign $y$ to $\cy_c$, then the probability of sampling a set $Y$ that is singular with respect to the new partition $\Pi$ is
\begin{align*}
  \pr[Y&\sim \kdpp{L}\text{ is singular}]\, =\, {1\over Z}\sum_{Y \text{$k$-singular}}\det(L_Y)\\
  &= {1\over Z} \Big(\sum_{Y \text{$k$-sing.}, y\notin Y} \det(L_Y) + \sum_{Y\text{$k$-sing.}, y\in Y} \det(L_Y)\Big)\\
  &= {1\over Z}\Big(\mathrm{const} + \sum_{Y'\,(k-1)\text{-sing. w.r.t $\Pi\setminus\cy_c$}} \det(L_{Y'\cup\{y\}})\Big)\\
  &= {1\over Z}\Big(\mathrm{const} + L_{yy} s_{k-1}^\Pi(L_{{\backprime} c}^{y})\Big),
\end{align*}
where $s_k^\Pi(L) := \sum_{Y\text{ $k$-sing.}} \det(L_Y)$. The matrix $L_{\backprime c}$ denotes $L$ with rows $\cy_c$ and columns $\cy_c$ deleted, and $L^{y} = L - L_{\cy,y}L_{y,\cy}$.
For local search, we would hence compute $L_{yy} s_{k-1}^\Pi(L_{{\backprime} c}^{y})$ for each point $y$ and core $c$, assign $y$ to the highest-scoring $c$. Since this testing is still expensive, we introduce further speedups in Section~\ref{sec:approx}.

\subsection{Constructing $\cc$}
When constructing $\cc$, we aim to minimize the upper bound on the total variation distance between $P_k$ and $\PC$ stated in Theorem~\ref{thm:varbound.main}. Since $\pns$ and $\epsm$ only depend on $\Pi$ and not on $\cc$, we here focus on minimizing $|1 - \frac{Z_{\cc}}{Z}|$, i.e., bringing $Z_{\cc}$ as close to $Z$ as possible. To do so, we again employ local search and subsequently swap each $c \in \cc$ with its best replacement $v \in \cy_c$. Let $\cc^{c,v}$ be $\cc$ with $c$ replaced by $v$. We aim to find the best swap 
\begin{align}
  v &= \argmin_{v\in\cy_c} |Z - Z_{\cc^{c,v}}| \\
  &= \argmin_{v\in\cy_c} |Z - e_{k}(L_{\cc^{c,v}(\cy)})|.
\end{align}
Computing $Z$ requires computing the coefficients $e_k(L)$, which takes a total of $O(N^3)$ time\footnote{In theory, this can be computed in $O(N^\omega\log(N))$ time~\cite{Biirgisser1997}, but the eigendecompositions and dynamic programming used in practice typically take cubic time.}. In the next section, we therefore consider a fast approximation.

\subsection{Faster constructions and further speedups}\label{sec:approx}
Local search procedures for optimizing $\Pi$ and $\cc$ can be further accelerated by a sequence of relaxations that we found to work well in practice (see Section~\ref{sec:exp}).
We begin with the quantity $s_{k-1}^\Pi(L_{{\backprime} c}^{y})$ that involves summing over sub-determinants of the large matrix $L$. Assuming the initialization is not too bad, we can use the current $\cc$ to approximate $\cy$. In particular, when re-assigning $y$, we substitute all other elements with their corresponding cores, resulting in the kernel $\widehat{L} = L_{\cc(\cy)}$. This changes our objective to finding the $c \in \cc$ that maximizes $s_{k-1}^\Pi(\widehat{L}_{\backprime c}^{y})$. Key to a fast approximation is now Lemma~\ref{lem:approx.main}, which follows from Lemma~\ref{lem:lemma_smpl.main}.
\begin{lemma}\label{lem:approx.main}
For all $k \leq |\Pi|$, it holds that 
  \begin{align*}
    s^\Pi_{k}(L_{\cc(\cy)}) = e_{k}(L_{\cc(\cy)}) = e_k(\tilde{L}).
  \end{align*}
\end{lemma}
  \vspace{-.1in}
  \begin{proof}
  \begin{align*}
  s^\Pi_{k}&(L_{\cc(\cy)}) = \sum_{Y\text{ $k$-sing.}} \det((L_{\cc(\cy)})_Y) = \sum_{Y\text{ $k$-sing.}} \det(L_{\cc(Y)})\\
  &= \sum_{|Y|=k}\det(L_{\cc(Y)}) = e_k(L_{\cc(\cy)}) = e_k(\Lc);
\end{align*}
the last equality was shown in the proof of~\refthm{thm:varbound.main}.
  \vspace{-.1in}
  \end{proof}
Computing the normalizer $e_k(\tilde{L})$ only needs $O(M^3)$ time. We refer to this acceleration as \algoz.

Second, when constructing $\cc$, we observed that $Z_{\cc}$ is commonly much smaller than $Z$. Hence, a fast approximation merely greedily increases $Z_{\cc}$ without computing $Z$.

Third, we can be lazy in a number of updates: for example, we only consider changing cores for the part that changes. When a part $\cy_c$ receives a new member, we  check whether to switch the current core $c$ to the new member.
  This reduction keeps the core adjustment at time $O(M^3)$.
Moreover, when re-assigning an element $y$ to a different part $\cy_c$, it is usually sufficient to only check a few, say, $\nu$ parts with cores closest to $y$, and not all parts. The resulting time complexity for each element is $O(M^3)$. 

\begin{algorithm}[t]\small
	\caption{Iterative construction of $\Pi$ and $\cc$}\label{dcsalgo}
	\begin{algorithmic} 
	\REQUIRE{$\Pi$ initial partition; $\cc$ initial coreset; $k$ the size of sampled subset; $\nu$ number of nearest neighbors taken into consideration}
	\WHILE{not converged}
		\FORALL{$y\in\cy$}
			\STATE $c\leftarrow$ group in which $y$ lies currently: $y \in \cy_c$
			\IF{$y\in\cc$}
				\STATE continue
			\ENDIF
			\STATE $G\leftarrow \{$groups of $\nu$ cores nearest to $X_y\}$
			\STATE $g^* = \argmax_{g\in G} s_{k-1}^\Pi(\widehat{L}_{\backprime g}^{y})$
			\IF{$c\ne g^*$}
			\STATE $\cy_c = \cy_c\backslash\{y\}$
			\STATE $\cy_{g^*} = \cy_{g^*}\cup\{y\}$
				\IF{$e_{k}(L_{\cc^{g^*,j}(\cy)}) > e_{k}(L_{\cc^{c,j}(\cy)})$}
					\STATE $\cc \leftarrow \cc^{g^*,y}$
				\ENDIF
			\ENDIF
		\ENDFOR
		\FORALL{$g\in[M]$}
			\STATE $j = \argmax_{j\in\cy_g} e_{k}(L_{\cc^{g,j}(\cy)})$
			\STATE $\cc = \cc^{g,j}$
		\ENDFOR
	\ENDWHILE
\end{algorithmic}
\end{algorithm}

With this collection of speedups, the approximate construction of $\Pi$ and $\cc$ takes $O(N M^3)$ for each iteration, which is linear in $N$, and hence a huge speedup over direct methods that require $O(N^3)$ preprocessing. The iterative algorithm is shown in Algorithm \ref{dcsalgo}. 
The initialization also affects the algorithm performance, and in practice we find that kmeans++ as an initialization works well. Thus we use \algo to refer to the algorithm that is initialized with kmeans++ and uses all the above accelerations. 
In practice, the algorithm converges very quickly, and most of the progress occurs in the first pass through the data. Hence, if desired, one can even use early stopping.

\section{Experiments}\label{sec:exp}
\vspace*{-4pt}
We next evaluate \algo, and compare its efficiency and effectiveness against three competing approaches:
\begin{list}{-}{\leftmargin=1.2em}
  \vspace*{-8pt}
  \setlength{\itemsep}{-2pt}
\item Partitioning using $k$-means (with kmeans++ initialization~\cite{Arthur2007}), with $\cc$ chosen as the centers of the clusters; referred to as {\it K++} in the results.
\item The adaptive, stochastic Nystr\"om sampler of~\cite{Kulesza2013} ({\it NysStoch}). 
  We used $M$ dimensions for NysStoch, to use the same dimensionality as \algo.
\item The Metropolis-Hastings DPP sampler \textit{MCDPP}~\cite{Kang2013}.  
  We use the well-known Gelman and Rubin multiple sequence diagnostic~\cite{gelman1992inference} to empirically judge mixing. 
\end{list}
In addition, we show results using different variants of \algo: \algoz described in~\refsec{sec:approx} and variants that are initialized either randomly~(\algor) or via kmeans++~(\algo).

\subsection{Synthetic Dataset}
\label{sec:synth}
We first explore the effect of our fast approximate sampling on controllable synthetic data.
The experiments here compare the accuracy of the faster \algo from Section~\ref{sec:approx} to \algoz, \algor and {\it K++}.

We generate an equal number of samples from each of {\it nClust} 30-dimensional Gaussians with means of varying length ($\ell_2$-norm) and unit variance, and then rescale the samples to have the same length. As the length of the samples increases, $\epsm$ and $\pns$ shrink. Finally, $L$ is a linear kernel.
	Throughout this experiment we set $k=4$ and $N=60$ to be able to exactly compute $\tvnorm{\Ph-P_k}$. We extract $M = 10$ core points and use $\nu = 3$ neighboring cores. Recall from~\refsec{sec:approx} that when considering the parts that one element should be assigned to, it is usually sufficient to only check $\nu$ parts with cores closest to $y$. 
Thus, $\nu = 3$ means we only consider re-assigning each element to its three closest parts.

\paragraph{Results.}

\begin{figure}
\begin{center}
\includegraphics[width=\textwidth]{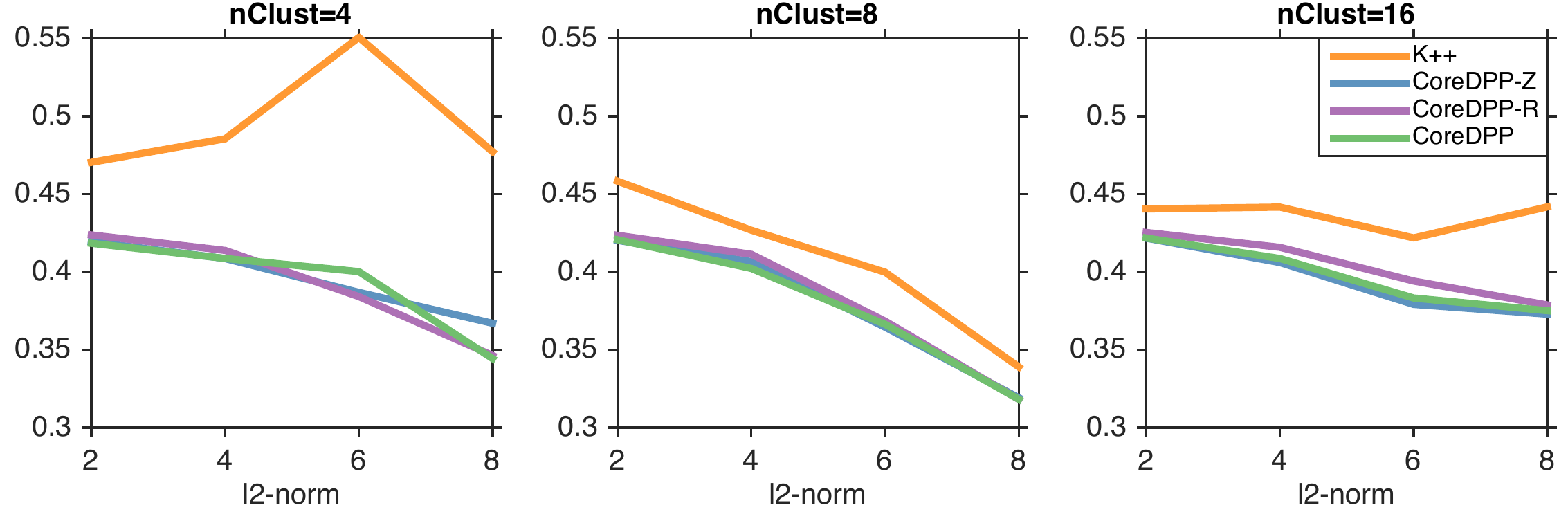}
\end{center}
\caption{Total variation distances (error) on synthetic data with varying {\it nClust} and $\ell_2$-norm.}
\label{fig:syn}
\end{figure}

Fig.~\ref{fig:syn} shows the total variation distance $\tvnorm{\Ph-P_k}$ defined in Equation~\eqref{eq:tv} for the partition and cores generated by {\it K++}, \algo, \algor and \algoz as {\it nClust} and the length vary.
We see that in general, most approximations improve as $\epsm$ and $\pns$ shrink. Remarkably, the \algo variants achieve much lower error than {\it K++}. Moreover, the results suggest that the relaxations from Section~\ref{sec:approx} do not noticeably increase the error in practice. Also, \algor performs comparable with \algo, indicating that our algorithm is robust against initialization.
Since, in addition, the \algo construction makes most progress in the first pass through the data, and the kmeans++ initialization yields the best performance, we use only one pass of \algo initialized with kmeans++ in the subsequent experiments.

\subsection{Real Data}
We apply \algo to two larger real data sets:
\begin{enumerate}
\item  MNIST~\cite{lecun1998gradient}. MNIST consists of images of hand-written digits, each of dimension $28\times 28$. 
\item GENES~\cite{Batmanghelich2014}. This dataset consists of different genes. Each sample in GENES corresponds to a gene, and the features are shortest path distances to 330 different hubs in the BioGRID gene interaction network.
\end{enumerate}
\vspace{-.1in}
For our first set of experiments on both datasets, we use a subset of 2000 data points and an RBF kernel to construct $L$. To evaluate the effect of model parameters on performance, we vary $M$ from 20 to 100 and $k$ from 2 to 8 and fix $\nu = 2$ (see Section~\ref{sec:synth} for an explanation of the parameters).
Larger-scale experiments on these datasets are reported in Section~\ref{sec:ls}.

\paragraph{Performance Measure and Results.}

\begin{figure}
\begin{center}
\includegraphics[width=0.8\textwidth]{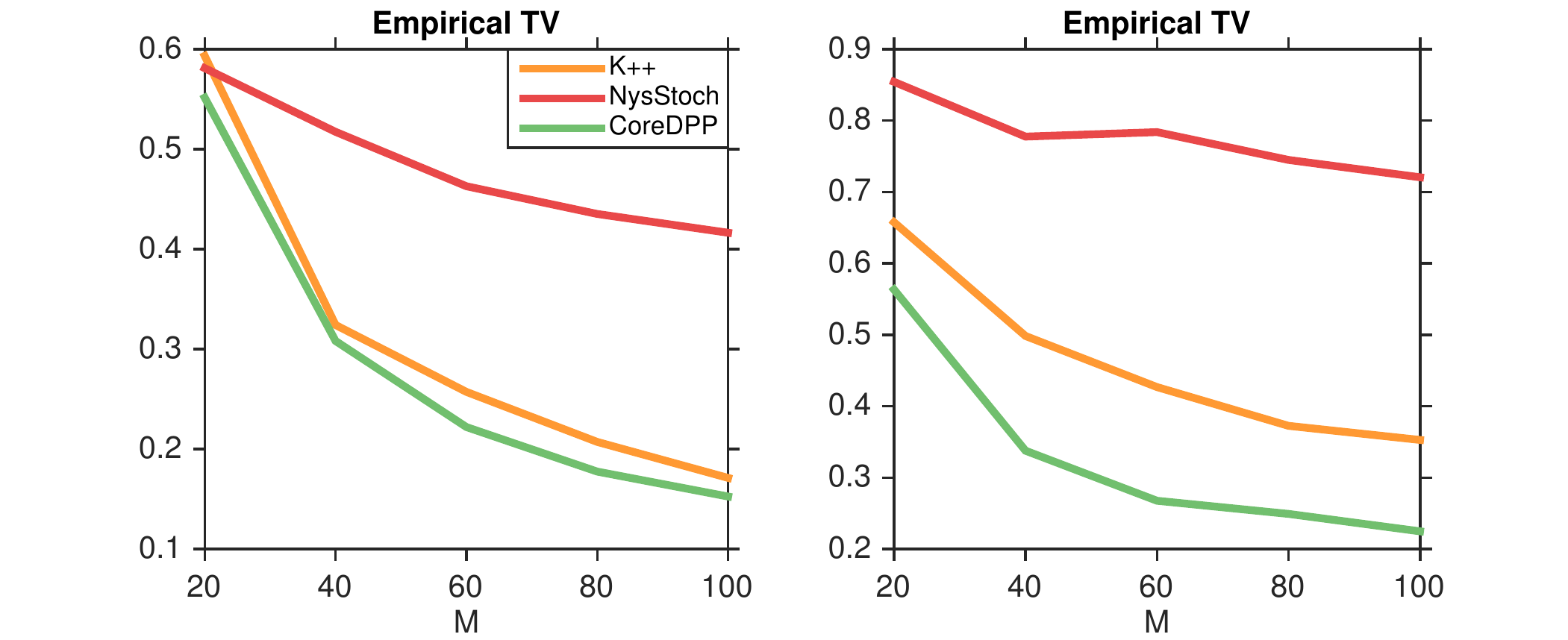}
\end{center}
\caption{Approximate total variation distances (empirical estimate) on MNIST~(left) and GENES~(right) with $M$ varying from 20 to 100 and fixed $k=6$.}
\label{real_perform}
\end{figure}

\begin{figure}
\begin{center}
\includegraphics[width=0.8\textwidth]{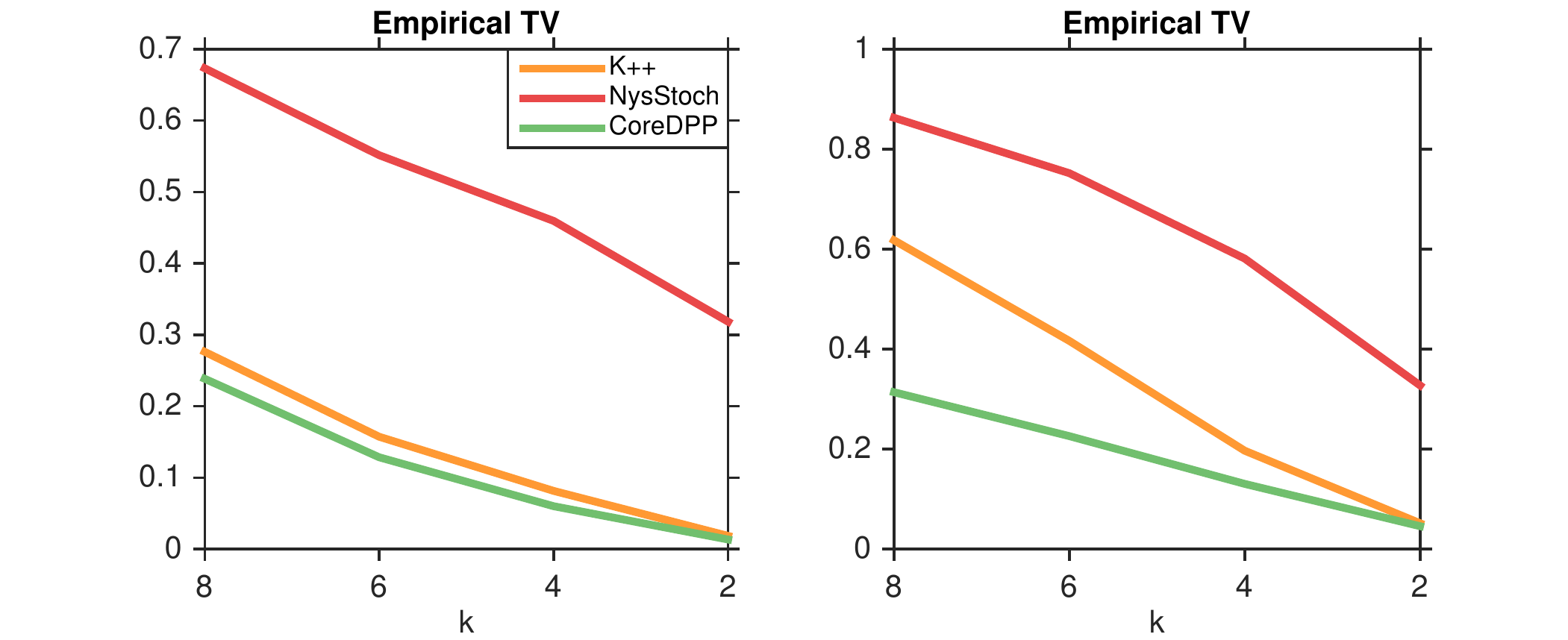}
\end{center}
\caption{Approximate total variation distances (empirical estimate) on MNIST~(left) and GENES~(right) with $k$ varying from 8 to 2 and fixed $M=100$.}
\label{real_perform_k}
\end{figure}

On these larger data sets, it becomes impossible to compute the total variation distance exactly. We therefore approximate it by uniform sampling and computing an empirical estimate. 

The results in Figure~\ref{real_perform} and Figure~\ref{real_perform_k} indicate that the approximations improve as the number of parts $M$ increases and $k$ decreases. This is because increasing $M$ increases the models' approximation power, and decreasing $k$ leads to a simpler target probability distribution to approximate. In general, \algo always achieves lower error than {\it K++}, and {\it NysStoch} performs poorly in terms of total variation distance to the original distribution. This phenomenon is perhaps not so surprising when recalling that the Nystr\"om approximation minimizes a different type of error, a distance between the kernel \emph{matrices}. These observations suggest to be careful when using matrix approximations to approximate $L$.

For an intuitive illustration, Figure~\ref{mnist_showcase} shows a core $\cc$ constructed by \algo, and the elements of one part $\mathcal{Y}_c$.

\begin{figure}
\begin{center}
\includegraphics[width=\textwidth]{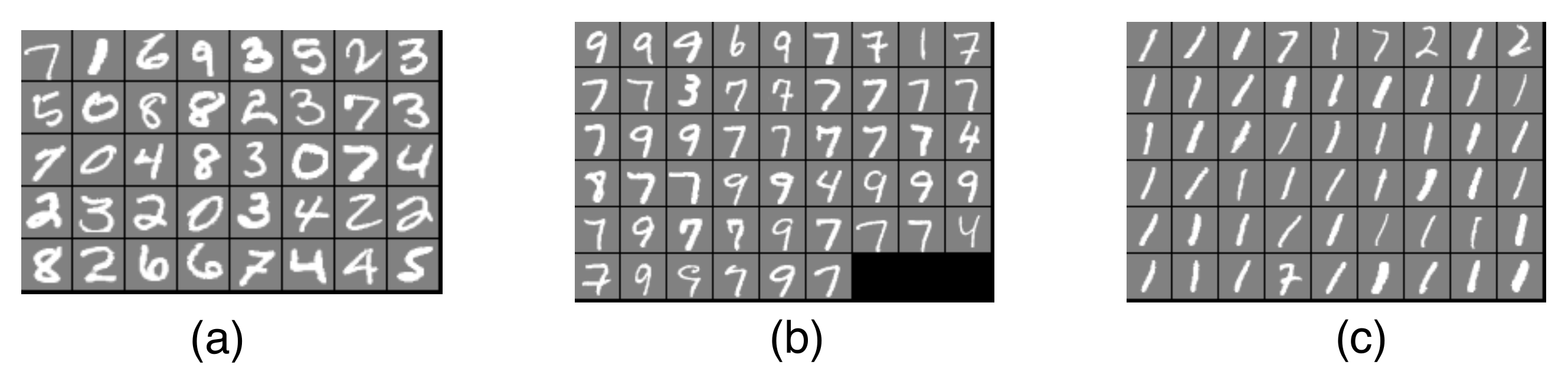}
\end{center}
\caption{(a) Example coreset $\cc$ of size 40, each figure is a core in the coreset constructed by our algorithm; (b,c) Two different parts corresponding to the first and second core.}
\label{mnist_showcase}
\end{figure}

\subsection{Running Time on Large Datasets}
\label{sec:ls}
Lastly, we address running times for \algo, {\it NysStoch} and the Markov chain $k$-DPP~({\it MCDPP}~\cite{Kang2013}). For the latter, we evaluate convergence  via the Gelman and Rubin multiple sequence diagnostic~\cite{gelman1992inference}; we run 10 chains simultaneously and use the {\it CODA}~\cite{coda} package to calculate the potential scale reduction factor~(PSRF), and set the number of iterations to the point when PSRF drops below 1.1. Finally we run {\it MCDPP} again for this specific number of iterations.

For overhead time, i.e., time to set up the sampler that is spent once in the beginning, we compare against {\it NysStoch}: \algo constructs the partition and $\Lc$, while {\it NysStoch} selects landmarks and constructs an approximation to the data. 
For sampling time, we compare against both {\it NysStoch} and {\it MCDPP}: \algo uses \refalgo{smplalgo}, and {\it NysStoch} uses the dual form of $k$-\dpp sampling~\cite{Kulesza2010}. 
We did not include the time for convergence diagnostics into the running time of {\it MCDPP}, giving it an advantage in terms of running time.

\begin{figure}
\begin{center}
\includegraphics[width=0.37\textwidth]{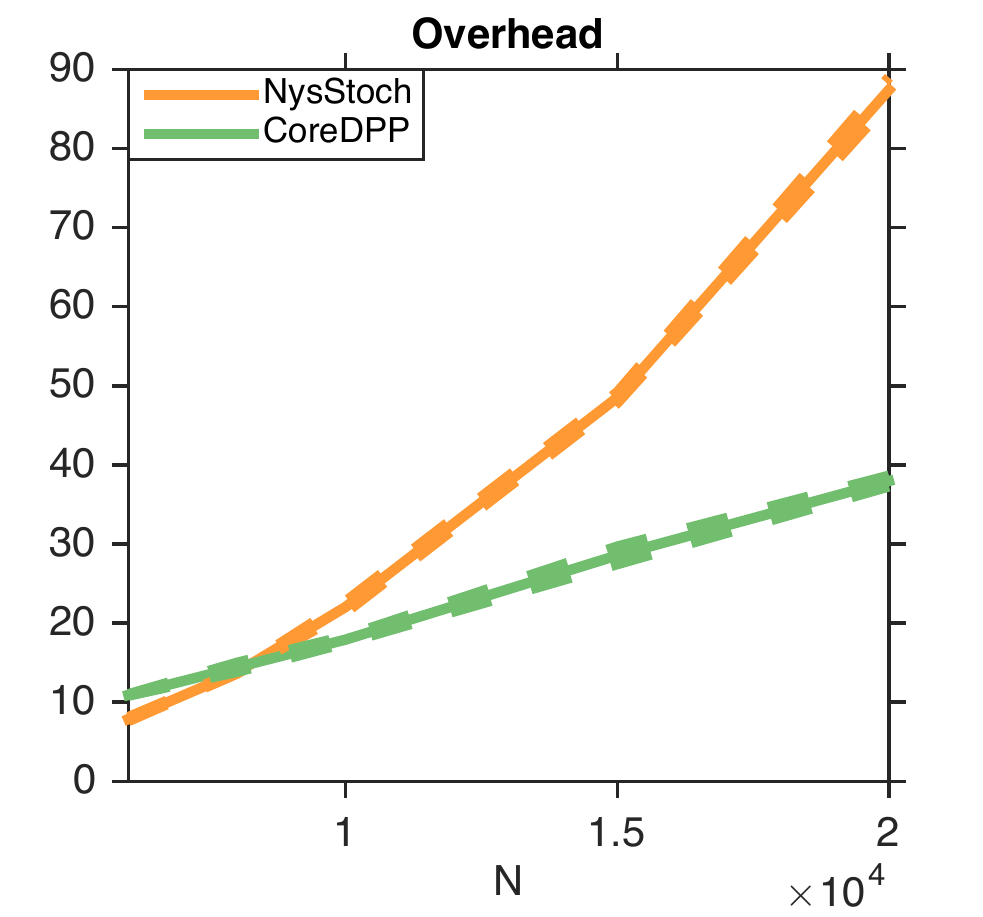}
\includegraphics[width=0.37\textwidth]{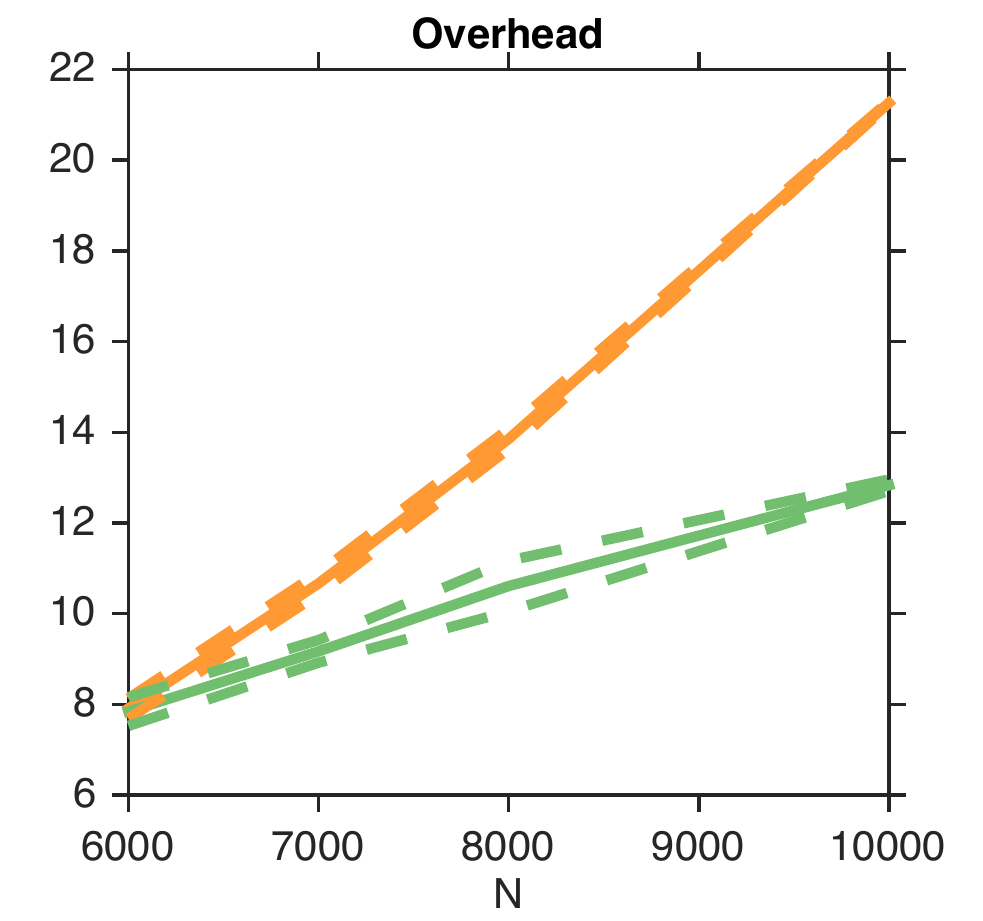}
\end{center}
\caption{Overhead (setup) time in seconds with varying ground set size~($N$) on MNIST~(left) and GENES~(right).}
\label{time_init}
\end{figure}

\paragraph{Overhead.}
\reffig{time_init} shows the overhead times as a function of $N$. For MNIST we vary $N$ from 6,000 to 20,000 and for GENES we vary $N$ from 6,000 to 10,000. These values of $N$ are already quite large, given that the \dpp kernel is a dense RBF kernel matrix; this leads to increased running time for all compared methods.  
The construction time for {\it NysStoch} and \algo is comparable for small-sized data, but {\it NysStoch} quickly becomes less competitive as the data gets larger. The construction time for \algo is linear in $N$, with a mild slope. If multiple samples are sought, this construction can be performed offline as preprocessing as it is needed only once. 

\paragraph{Sampling.}
\reffig{time_smpl} shows the time to draw one sample as a function of $N$, comparing \algo against {\it NysStoch} and {\it MCDPP}. \algo yields samples in time independent of $N$ and is extremely efficient -- it is orders of magnitude faster than {\it NysStoch} and {\it MCDPP}. 

\begin{figure}
\begin{center}
\includegraphics[width=0.38\textwidth]{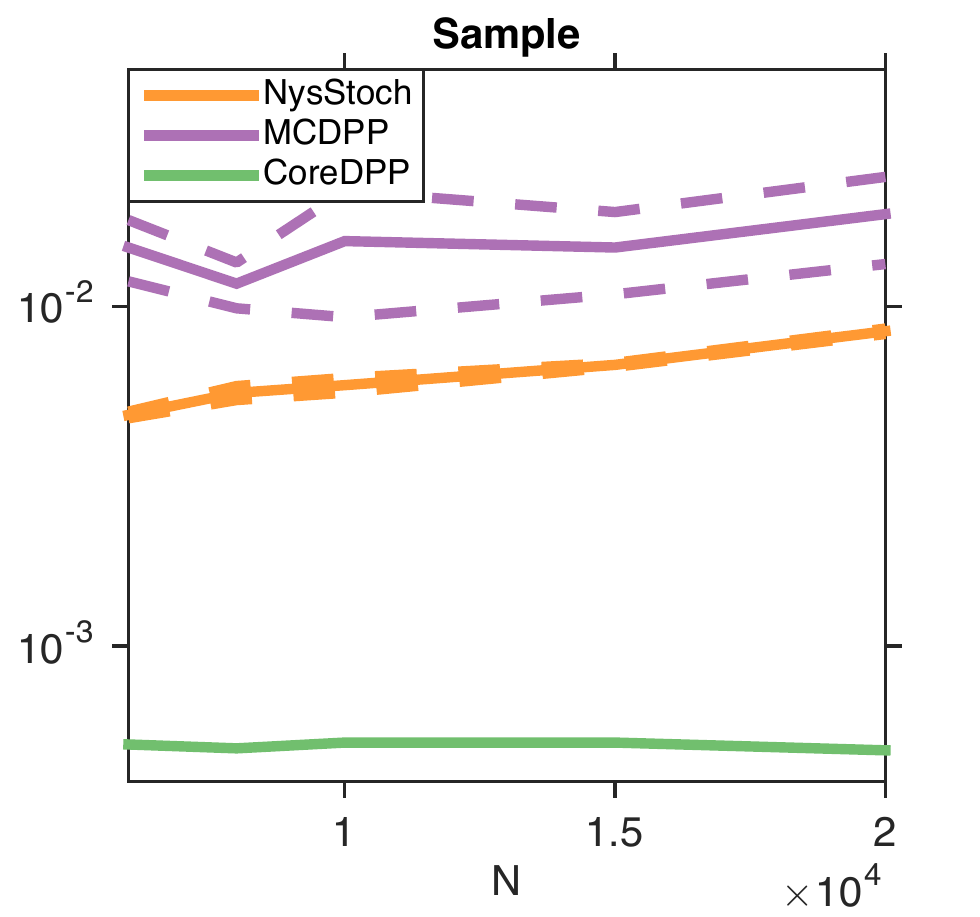}
\includegraphics[width=0.38\textwidth]{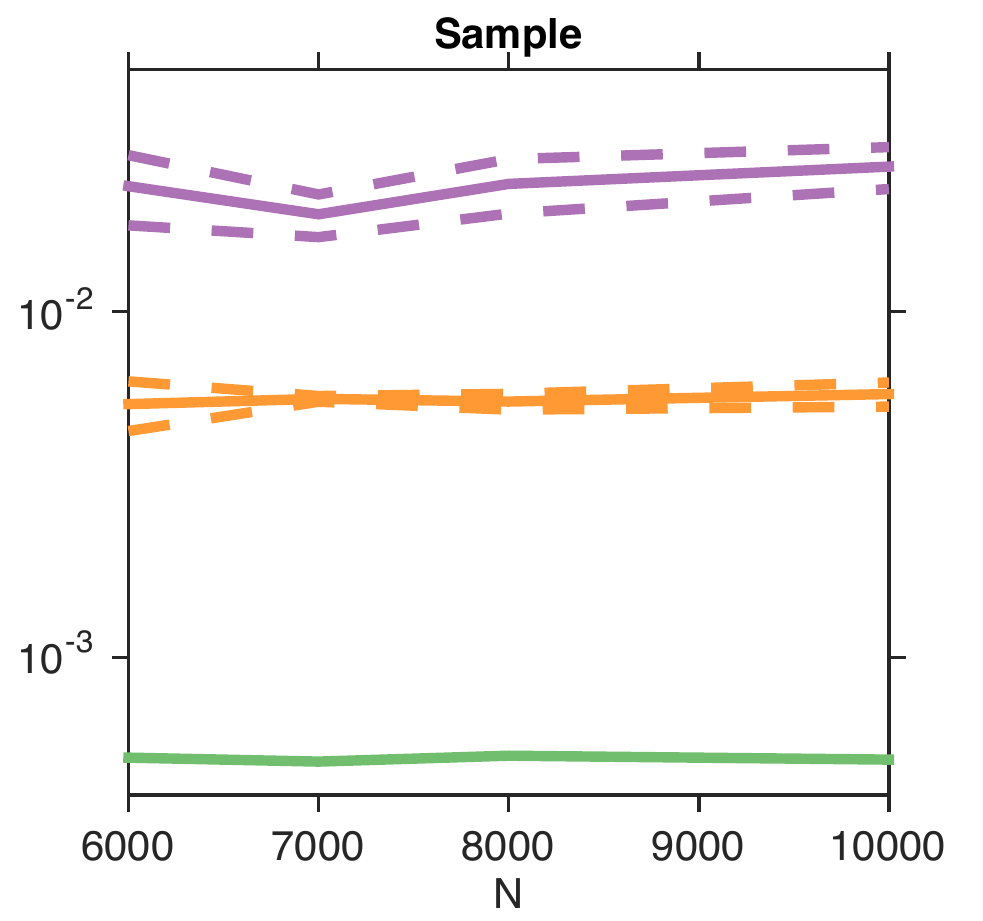}
\end{center}
\caption{Average time for drawing one sample as the ground set size~($N$) varies on MNIST~(left) and GENES~(right). Note that the time axis is shown in log scale.}
\label{time_smpl}
\end{figure}

We also consider the time taken to sample a large number of subsets, and compare against both {\it NysStoch} and {\it MCDPP}---the sampling times for drawing approximately independent samples with {\it MCDPP} add up. \reffig{time_tot} shows the results. As more samples are required, \algo becomes increasingly efficient relative to the other methods.

\begin{figure}
\begin{center}
\includegraphics[width=0.38\textwidth]{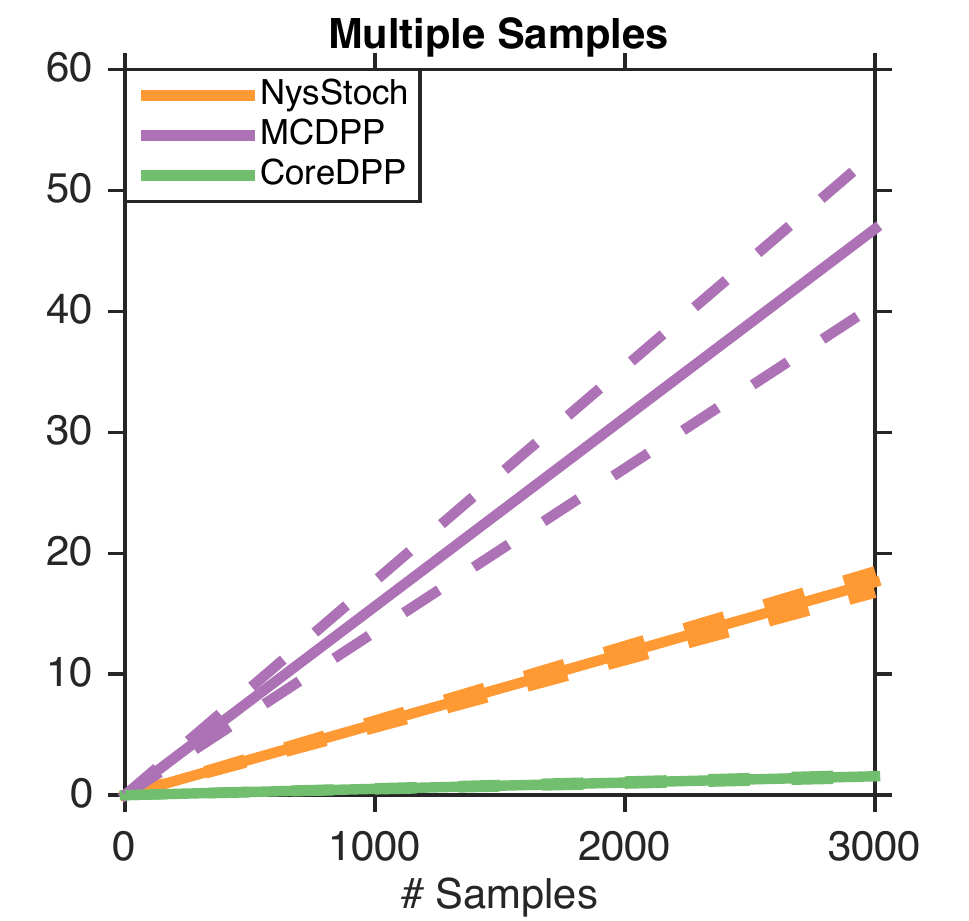}
\includegraphics[width=0.38\textwidth]{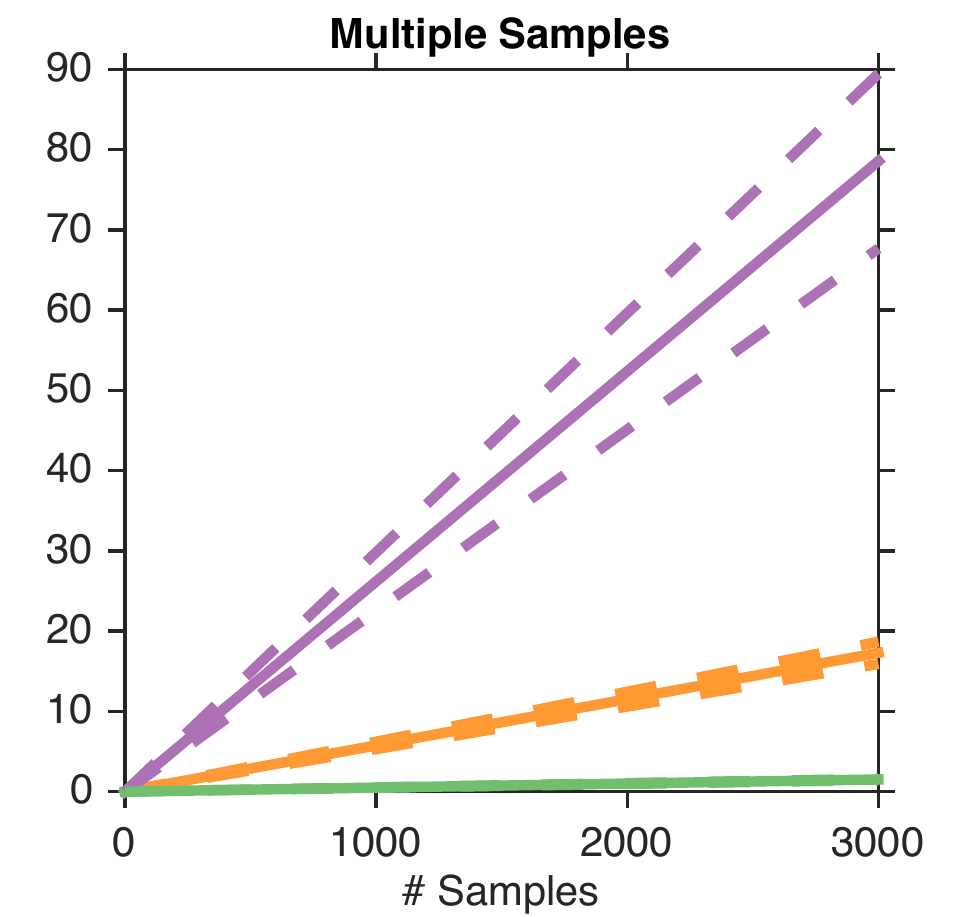}
\end{center}
\caption{Average time for sampling different numbers of subsets with $N=5000$, $M = 40$ and $k = 5$ on MNIST~(left) and GENES~(right). }
\label{time_tot}
\end{figure}

\section{Conclusion}
In this paper, we proposed a fast, two-stage sampling method for sampling diverse subsets with $k$-\dpp{}s. As opposed to other approaches, our algorithm directly aims at minimizing the total variation distance between the approximate and original probability distributions. Our experiments demonstrate the effectiveness and efficiency of our approach: not only does our construction have lower error in total variation distance compared with other methods, it also produces these more accurate samples efficiently, at comparable or faster speed than other methods. 

\paragraph{Acknowledgements.}
This research was partially supposed by an NSF CAREER award 1553284 and a Google Research Award.

\bibliographystyle{abbrvnat}

\end{document}